\documentclass[letterpaper,journal]{IEEEtran}

\usepackage[utf8]{inputenc}
\usepackage{array}
\usepackage{amsmath,amsthm,amssymb,epsf}
\usepackage{bm} 
\usepackage{graphicx,psfrag}
\usepackage{epstopdf}
\usepackage{verbatim}
\usepackage{multirow}
\usepackage{color}
\usepackage{soul,xcolor}
\usepackage{todonotes}
\usepackage{setspace}
\usepackage{enumitem}
\usepackage{subfigure}
\usepackage{comment}
\usepackage{tabularx}
\usepackage{booktabs}
\usepackage{dirtytalk}
{
      \theoremstyle{plain}
      \newtheorem{assumption}{Assumption}
  }
\usepackage{csquotes}

\newtheorem{theorem}{Theorem}

\newtheorem{remark}{Remark}

\usepackage{cite}

\usepackage{algorithmicx}
\usepackage[ruled,vlined, linesnumbered]{algorithm2e}

\usepackage{amsfonts}
\usepackage[caption=false,font=normalsize,labelfont=sf,textfont=sf]{subfig}
\usepackage{textcomp}
\usepackage{url}
\makeatletter
\newcommand{\removelatexerror}{\let\@latex@error\@gobble}

\IEEEoverridecommandlockouts

\begin{document}

\title{Data-Enabled Neighboring Extremal: Case Study on Model-Free Trajectory Tracking for Robotic Arm}

\author{Amin~Vahidi-Moghaddam, Keyi~Zhu, Kaixiang~Zhang, Ziyou~Song, and Zhaojian~Li$^*$
\thanks{$^*$Zhaojian Li is the corresponding author.}
\thanks{This work was supported by the U.S. National Science Foundation
Award CMMI-2320698.}
\thanks{Amin Vahidi-Moghaddam, Keyi Zhu, Kaixiang Zhang, and Zhaojian Li are with
the Department of Mechanical Engineering, Michigan State University, East
Lansing, MI 48824 USA (e-mail: vahidimo@msu.edu, zhukeyi1@msu.edu, zhangk64@msu.edu,
and lizhaoj1@egr.msu.edu).}
\thanks{Ziyou Song is with the Department of Electrical and Computer Engineering, University of Michigan, Ann Arbor, MI 48109 USA (e-mail: ziyou@umich.edu).}
}



\maketitle

\begin{abstract}
Data-enabled predictive control (DeePC) has recently emerged as a powerful data-driven approach for efficient system controls with constraints handling capabilities. It performs optimal controls by directly harnessing input-output (I/O) data, bypassing the process of explicit model identification that can be costly and time-consuming. However, its high computational complexity, driven by a large-scale optimization problem (typically in a higher dimension than its model-based counterpart--Model Predictive Control), hinders real-time applications. To overcome this limitation, we propose the data-enabled neighboring extremal (DeeNE) framework, which significantly reduces computational cost while preserving control performance. DeeNE leverages first-order optimality perturbation analysis to efficiently update a precomputed nominal DeePC solution in response to changes in initial conditions and reference trajectories. We validate its effectiveness on a 7-DoF KINOVA Gen3 robotic arm, demonstrating substantial computational savings and robust, data-driven control performance.
\end{abstract}


\section{Introduction}
\IEEEPARstart{S}{afe} and high-precision trajectory tracking for autonomous systems has traditionally relied on model-based control synthesis, which depends on accurate dynamical models~\cite{liu2022mpc}. 
Model predictive control (MPC), for instance, has been widely successful in ensuring both tracking performance and system safety across various applications. However, its reliance on precise models and substantial onboard computational resources limits its broader adoption in autonomous systems~\cite{li2024identification}. This challenge is particularly pronounced for compliant or low-cost robots/vehicles performing high-precision tasks or operating in rapidly changing environments, such as construction sites where robots/vehicles are exposed to severe weather and working conditions \cite{zhu2018high}. 
To mitigate computational complexity, researchers have explored control policy learning approaches \cite{salzmann2023real,williams2017information,vahidi2023unified} and reduced-order modeling techniques \cite{zhang2023dimension,zamani2024data,ashtiani2022scalable}, though these methods often require extensive data collection or involve trade-offs between system performance and computational efficiency. 

On the other hand, data-driven optimal controllers have gained considerable attention in both academia and industry, as they eliminate the requirement for an explicit system model, which can be costly and time-consuming to develop. For example, in robot manipulation, deep learning techniques have been applied to learn robot kinematics, including both forward and inverse kinematics~\cite{sharkawy2022forward,singh2024application,toquica2021analytical}. Reinforcement learning and iterative learning control are also prominent methods for learning-based control in robotic manipulators~\cite{lu2024robust,anand2023model,ngo2024robust}. Gaussian Process Regression for robotics tasks has been outlined in \cite{calandra2016manifold} and extended to MPC for robotic manipulators~\cite{carron2019data}. To overcome the need for explicit models and reduce computational time, spatial-temporal filters have been used for MPC policy learning \cite{vahidi2023unified}. Recently, a new family of neural networks – neural ordinary differential equations (NODE) -  has emerged as an effective tool for extracting dynamic models from data \cite{chen2018neural}, as it approximates a continuous-depth neural network to directly model differential equations. This algorithm has been extended to a NODE-based MPC framework for aerial robots \cite{chee2022knode}. However, while data-driven controllers offer significant advantages, they often require extensive training data to achieve reliable control performance, and their lack of interpretability remains a key challenge.

Data-enabled predictive control (DeePC) has recently emerged as a new powerful data-driven optimal control approach, transitioning from traditional MPC to a model-free framework~\cite{coulson2019data}. By leveraging raw input/output (I/O) data, DeePC directly seeks an optimal control policy without requiring prior system identification. Building on the Fundamental Lemma \cite{willems2005note} and inspired by behavioral system theory \cite{willems1997introduction}, DeePC offers a novel approach to learning, predicting, and controlling system behavior by representing the subspace of the I/O trajectories as the column span of a data Hankel matrix. Regularization techniques and slack variables can be incorporated to address overfitting and infeasibility issues \cite{NL-DEEPC1,NL-DEEPC2}, which can arise from noisy data and nonlinear systems. In comparison with the aforementioned machine learning-based controllers, DeePC is less data hungry~\cite{coulson2019data}. More importantly, DeePC can explicitly handle constraints, which are crucial in many safety-critical real-world engineering systems.

Despite its promise, DeePC often suffers from high computational complexity due to the large dimensionality of the optimization variable, which is generally higher than that of its model-based MPC counterpart \cite{zhang2023dimension}. To address this challenge, this paper introduces a computationally efficient data-driven control framework, data-enabled neighboring extremal (DeeNE), which significantly enhances the computational efficiency of DeePC with minimum or no degradation in control performance.  First-order optimality perturbation analysis on a nominal DeePC solution is conducted to derive a feedback correction law that adapts to I/O perturbations. Compared to our previous conference version~\cite{vahidi2023data}, this work extends the DeeNE framework to incorporate real-time reference trajectory perturbations, enabling adaptive updates to the control policy for varying desired trajectories. Furthermore, the framework is validated through both simulation and experimental studies on a 7-DoF KINOVA Gen3 robotic arm, demonstrating substantial improvements in computational efficiency and constraint-handling capabilities in real-world applications.


The remainder of the paper is organized as follows: Section II introduces the problem formulation and reviews the DeePC basics. Section III presents the main results on the DeeNE framework. Simulation and experimental results are shown in Section IV. Finally, Section V concludes the paper.

\textbf{Notations}. We adopt the following notations across the paper. $\mathbb{R}^n$ and $\mathbb{R}^{n\times m}$ represent the set of $n$-dimensional real vectors and the set of $n\times m$-dimensional real matrices, respectively. $x^{\top}$ and $A^{\top}$ stand for the transpose of the vector $x$ and the matrix $A$, respectively. $\lVert \cdot \rVert$ denotes the Euclidean norm of a vector or the induced $2$-norm of a matrix. $I_n$ stands for the identity matrix with $n$-dimension. $J_g$, $J_{gg}$, $J_{gr}$, and $J_{rg}$ represent $\frac{\partial J}{ \partial g}$, $\frac{\partial^2 J }{ \partial g^2}$, $\frac{\partial^2 J}{\partial g \partial r}$, and $\frac{\partial^2 J}{ \partial r \partial g}$, respectively.

\section{Problem Formulation and Preliminaries}
In this section, we introduce the problem of reference tracking for nonlinear systems with constraints, followed by the model-based optimal control formulation, which is called model predictive control (MPC). We then present the data-driven predictive control formulation that bypasses the dynamic modeling process, which is called data-enabled predictive control (DeePC).

\subsection{Model-Based Reference Tracking}
Consider the following discrete-time nonlinear system:
\begin{equation}
  \begin{aligned}
    \label{system}
     &x(k+1) = f(x(k),u(k)),\\
     &y(k) = h(x(k),u(k)),
  \end{aligned}
\end{equation}
where $k\in\mathbb{N}^+$ denotes the time step, $x\in\mathbb{R}^n$ represents the state vector of the system, $u \in {\mathbb{R}^m}$ is the control input, and $y \in \mathbb{R}^p$ denotes the outputs of the system. Moreover, $f:\mathbb{R}^n\times \mathbb{R}^m \rightarrow \mathbb{R}^n$ is the system dynamics with $f(0,0)=0$, and $h:\mathbb{R}^n\times \mathbb{R}^m \rightarrow \mathbb{R}^p$ represents the output dynamics.

We consider the following safety constraint:
\begin{equation}
    \label{safety}
  C(y(k),u(k)) \leq 0,
\end{equation}
where $C:\mathbb{R}^p\times \mathbb{R}^m \rightarrow \mathbb{R}^l$ with $l$ denoting the total number of constraints in inputs and outputs. 

We consider a tracking control problem for the nonlinear system \eqref{system} with a desired output reference trajectory $r(k),k=0,1,2,\cdots$. Such a problem can be solved using a receding horizon optimal control, also known as model predictive control (MPC) \cite{liu2022mpc}, by minimizing the following cost term over a prescribed horizon of $N$ steps:
\begin{equation}
    \label{cost}
    \begin{aligned}
  &J_N(\mathbf{y},\mathbf{u};\mathbf{r}) = \sum^{N-1}_{k=0} \phi(y(k),u(k);r(k)),
  \end{aligned}
\end{equation}
where $\mathbf{u} = [ u(0),\, u(1),\, \cdots,\, u(N-1)  ]$, $\mathbf{y} = [ y(0),\, y(1),\, \cdots,\, y(N-1)  ]$, $\mathbf{r} = \left[ r(0),\, r(1),\, \cdots,\, r(N-1)  \right]$, and $\phi(y,u;r)$ is the stage cost, which generally includes the tracking error and the control efforts. 
Therefore, with an initial state $x_o$, the optimal tracking problem over $N$ steps can be reduced to the following constrained optimization problem:
\begin{equation}
  \begin{aligned}
    \label{NMPC}
    &(\mathbf{y}^{*},\mathbf{u}^{*}) = \underset{\mathbf{y},\mathbf{u}}{\arg\min} \hspace{1 mm} J_N(\mathbf{y},\mathbf{u};\mathbf{r})\\
    &s.t. \hspace{5 mm} x(k+1) = f(x(k),u(k)),\\
    & \hspace{10 mm} y(k) = h(x(k),u(k)),\\
    & \hspace{10 mm} C(y(k),u(k)) \leq 0,\quad x(0)=x_o.
  \end{aligned}
\end{equation}

In the MPC setting, only the first optimal control (i.e., $u(0)$) is executed, and the system evolves one step. The process is then repeated with the current state of the system as the new initial state for MPC. 
The above MPC formulation \eqref{NMPC} has been a celebrated and widely used control technique for trajectory tracking, due to its capability to enforce safety constraints during the control design. The key ingredient for this controller is an accurate parametric model of the system, but obtaining such a model, using plant modeling or identification procedures, is often the most time-consuming and costly part of control design. We next introduce its data-driven counterpart, DeePC, which is able to enforce safety constraints on the control design and performs similar predictive control without the need for an explicit dynamical model.

\subsection{Data-Enabled Predictive Control (DeePC)}
DeePC has recently emerged as a popular data-driven optimal control framework that has achieved numerous successes. It directly harnesses the raw input and output data for controls, eliminating the requirement of an explicit dynamic model as used in \eqref{NMPC}.  Building on the Fundamental Lemma \cite{willems2005note}, DeePC leverages a non-parametric representation of the dynamic system following the behavioral system theory \cite{willems1997introduction}. Specifically, Hankel matrices $\mathbb{H}_{T_{ini}+N}(u^d)$ and $\mathbb{H}_{T_{ini}+N}(y^d)$, where $T_{ini} + N$ represents the depth of the Hankel matrix, are first constructed from the offline collected input/output (I/O) samples $u^d$ and $y^d$ as:
\begin{equation}
\small
  \begin{aligned}
    \label{Hankel u}
    & \mathbb{H}_{T_{ini}+N}(u^d) = \begin{bmatrix}
    u^d_1 & u^d_2 & \cdots & u^d_{T-T_{ini}-N+1}\\ 
    u^d_2 & u^d_3 & \cdots & u^d_{T-T_{ini}-N+2}\\
    \vdots & \vdots & \ddots & \vdots\\ 
    u^d_{T_{ini}+N} & u^d_{T_{ini}+N+1} & \cdots & u^d_T
    \end{bmatrix},
  \end{aligned}
\end{equation}
where $T_{ini}$, $N$, and $T$ denote the length of the initial trajectory, prediction trajectory, and collected data, respectively. $\mathbb{H}_{T_{ini}+N}(u^d) \in \mathbb{R}^{m(T_{ini}+N) \times L}$, $L = T - T_{ini} - N + 1$, and $\mathbb{H}_{T_{ini}+N}(y^d) \in \mathbb{R}^{p(T_{ini}+N) \times L}$ is constructed in an analogous way from the collected samples $y^d$. It should be mentioned that persistency of excitation requirement \cite{coulson2019data} is generally needed for the signal $u^d$, which can be met if $\mathbb{H}_{T_{ini}+N+n}(u^d) \in \mathbb{R}^{m(T_{ini}+N+n) \times L}$ has full row rank.

One can partition the Hankel matrices into \textit{past} and \textit{future} sub-blocks as:
\begin{equation}
  \begin{aligned}
    \label{Partitioned Hankel}
    &\begin{bmatrix}
      U_P\\ U_F 
     \end{bmatrix} =: \mathbb{H}_{T_{ini}+N}(u^d), \hspace{5 mm}
     \begin{bmatrix}
      Y_P\\ Y_F 
     \end{bmatrix} =: \mathbb{H}_{T_{ini}+N}(y^d),
  \end{aligned}
\end{equation}
where $U_P \in \mathbb{R}^{mT_{ini} \times L}$, $U_F \in \mathbb{R}^{mN \times L}$, $Y_P \in \mathbb{R}^{pT_{ini} \times L}$, and $Y_F \in \mathbb{R}^{pN \times L}$. Now, DeePC aims at optimizing the system performance over $N$ future steps using only the I/O data \eqref{Partitioned Hankel}, which is presented as \cite{coulson2019data,huang2023robust}:
\begin{equation}
\centering
  \begin{aligned}
    \label{DeePC}
    &(\mathbf{y}^{*},\mathbf{u}^{*},\mathbf{\sigma_y}^{*},\mathbf{\sigma_u}^{*},\mathbf{g}^{*}) = \underset{\mathbf{y},\mathbf{u},\mathbf{\sigma_y},\mathbf{\sigma_u},\mathbf{g}}{\arg\min} \hspace{1 mm} J_N(\mathbf{y},\mathbf{u},\mathbf{\sigma_y},\mathbf{\sigma_u},\mathbf{g},\mathbf{r})\\
    &s.t. \hspace{5 mm} \begin{bmatrix}
      U_P\\ Y_P\\ U_F\\ Y_F 
     \end{bmatrix} g = 
     \begin{bmatrix}
      u_{ini}\\ y_{ini}\\ u\\ y 
     \end{bmatrix} +
     \begin{bmatrix}
      \sigma_u\\ \sigma_y\\ 0\\ 0 
     \end{bmatrix},\\
     &\hspace{11.5 mm} C(y,u) \leq 0,
  \end{aligned}
\end{equation}
where the equality constraint is a result of Fundamental lemma with $\sigma_u \in \mathbb{R}^{mT_{ini}}$ and $\sigma_y \in \mathbb{R}^{pT_{ini}}$ being auxiliary slack variables to model measurement noises and nonlinearities, and $(u_{ini}, y_{ini})$ is the given initial trajectory. Moreover, $J_N(\mathbf{y},\mathbf{u},\mathbf{\sigma_y},\mathbf{\sigma_u},\mathbf{g},r)$ is the modified cost function including two penalty terms for slack variables and also a regularization term to avoid overfitting caused by noisy data and/or system nonlinearity.

Using the identities $y = Y_F g$, $u = U_F g$, $\sigma_y = Y_P g - y_{ini}$, and $\sigma_u = U_P g - u_{ini}$, one can rewrite \eqref{DeePC} as:
\begin{equation}
  \begin{aligned}
    \label{DeePC g}
    &\mathbf{g}^{*} = \underset{\mathbf{g}}{\arg\min} \hspace{1 mm} J_N(Y_F\mathbf{g},U_F\mathbf{g},Y_P \mathbf{g} - y_{ini},U_P \mathbf{g} - u_{ini},\mathbf{g},\mathbf{r})\\
    &s.t. \hspace{5 mm} C(Y_F g,U_F g) \leq 0.
  \end{aligned}
\end{equation}

If the constraint $C(y,u)$ was absent in \eqref{DeePC}, the problem is referred to the unconstrained DeePC, and the solution is available in closed form $u = U_F g = K_d^r r + K_d^{ini} w_{ini}$, where $K_d^r \in \mathbb{R}^{mN \times pN}$ and $K_d^{ini} \in \mathbb{R}^{mN \times (m+p)T_{ini}}$ are control gains, and $r$ and $w_{ini} = \begin{bmatrix} u_{ini}^T, y_{ini}^T \end{bmatrix}^T$ are the given desired reference trajectory and initial trajectory. However, in the general case with system constraints, the DeePC-induced optimization \eqref{DeePC} generally suffers from high computational time since the dimension of the optimization variable $g$ is exceedingly high to ensure the persistency of excitation requirement \cite{zhang2023dimension}. In the next section, we present a computationally efficient DeePC framework to mitigate the computation complexity.

\section{Data-Enabled Neighboring Extremal (DeeNE)}
In this section, inspired by neighboring extremal framework \cite{ghaemi2009neighboring,ghaemi2010robust}, we propose a new DeePC framework, named data-enabled neighboring extremal (DeeNE), to significantly reduce computational cost with no or little performance degradation. Specifically, consider the nominal solution $(g^o,u^o,y^o)$ under the given initial I/O trajectory $w^o_{ini}$ and reference trajectory $r^o$ in the DeePC formulation~\eqref{DeePC}. Now for a new initial I/O trajectory $w_{ini}$ and/or reference trajectory $r$, instead of rerunning the computationally expensive optimization in ~\eqref{DeePC}, we seek an optimal  feedback policy for correcting the DeePC solution as $u^* = u^{o} + \delta u$ using perturbation analysis. The objective is now to present the derivation of the feedback gains, corresponding to the perturbations $\delta w_{ini}$ and $\delta r$, for the optimal correction policy $\delta u$. We will first consider the case of the nominal trajectory obtained by solving ~\eqref{DeePC} (i.e., optimal nominal trajectory) and then consider the case of a non-optimal nominal trajectory.

\subsection{DeeNE with Optimal Nominal Solution}
Following \eqref{DeePC g}, we first construct the following augmented cost function:
\begin{equation}
  \begin{aligned}
    \label{aug-cost}
  &\bar{J}_N (g,w_{ini},r,\mu) = J_N (g,w_{ini},r) + \mu^{T} C^a (g),
    \end{aligned}
\end{equation}
where $C^a (g)$ represents the active constraints at the nominal solution, and $\mu$ is the Lagrange multiplier vector associated with the active constraints. Let $(g^o,w^o_{ini},r^o)$ represent nominal solution of DeePC \eqref{DeePC g}, which must satisfy the following necessary optimality conditions (KKT conditions):
\begin{equation}
    \label{KKT DeePC g}
    \bar{J}_g (g,w_{ini},r,\mu) = 0, \qquad \mu \geq 0,
\end{equation}
where the subscript $g$ indicates the partial derivative $\partial / \partial g$; thus, $\bar{J}_g$ represents $\partial \bar{J}_N / \partial g$, where for simplicity, we show $\bar{J_N}_g$ as $\bar{J}_g$.

\begin{assumption}
\label{ass2}
$C^a_{g}(g)$, i.e., $\partial C^a / \partial g$, is full row rank.
\end{assumption}

Substituting the nominal solution $(g^o,w^o_{ini},r^o)$ into the KKT conditions \eqref{KKT DeePC g} yields
\begin{equation}
  \begin{aligned}
    \label{Optimal KKT}
    &J_g (g^o,w^o_{ini},r^o) + \mu^{T} C^a_g (g^o) = 0.
  \end{aligned}
\end{equation}

The Lagrange multiplier can thus be obtained online as:
\begin{equation}
  \begin{aligned}
    \label{Lagrange multipliers}
    &\mu = -{(C^a_{g} {C^a_{g}}^T)}^{-1} C^a_{g} J^T_g.
  \end{aligned}
\end{equation}

Note that Assumption \ref{ass2} guarantees that $C^a_{g} {C^a_{g}}^T$ is invertible. Moreover, it is worth noting that $\mu = 0$ if the constraint $C (g^o)$ is not active. The Lagrange multiplier \eqref{Lagrange multipliers} is considered as the nominal Lagrange multiplier $\mu^{o}$.

Now, consider the nominal solution $(g^o,u^o,y^o,\mu^{o})$ under the given $w^o_{ini}$ and $r^o$. For a new initial I/O trajectory $w_{ini}=w^o_{ini}+\delta w_{ini}$ and a reference trajectory $r=r^o+\delta r$, with $\delta w_{ini}$ and $\delta r$ denoting deviations from nominal conditions, the objective is to develop the DeeNE framework which achieves $\delta u$ by minimizing the second-order variation of \eqref{aug-cost} subject to linearized constraints. More specifically, for the given $\delta w_{ini}$ and $\delta r$, DeeNE solves the following optimization problem:
\begin{equation}
  \begin{aligned}
    \label{DeeNE}
    &\mathbf{\delta g^{*}} = \underset{\mathbf{\delta g}}{\arg\min} \hspace{1 mm} {J}^{ne}_N\\ 
    & s.t. \hspace{5 mm} C^{a}_{g} \delta g = 0,
  \end{aligned}
\end{equation}
where 
\begin{equation}
  \begin{aligned}
    \label{DeeNE-cost}
  &{J}^{ne}_N(\delta g,\delta w_{ini},\delta r)=\delta^{2} \bar{J}_N(g^o,w^o_{ini},r^o) = \\
  &\frac{1}{2} 
  \begin{bmatrix}
  \delta g\\ \delta w_{ini}\\ \delta r
  \end{bmatrix}^{T}
  \begin{bmatrix}
  \bar{J}_{gg} & \bar{J}_{gw_{ini}} & \bar{J}_{gr}\\
  \bar{J}_{w_{ini}g} & \bar{J}_{w_{ini}w_{ini}} & \bar{J}_{w_{ini}r}\\
  \bar{J}_{rg} & \bar{J}_{rw_{ini}} & \bar{J}_{rr}
  \end{bmatrix}
  \begin{bmatrix}
  \delta g\\ \delta w_{ini}\\ \delta r
  \end{bmatrix}.
    \end{aligned}
\end{equation}
where the subscripts $gg$ and others represent the second-order partial derivatives.

For \eqref{DeeNE}, the augmented cost function is obtained as
\begin{equation}
  \begin{aligned}
    \label{DeeNE-aug-cost}
    &\bar{J}^{ne}_N(\delta g,\delta w_{ini},\delta r,\delta \mu) = {J}^{ne}_N(\delta g,\delta w_{ini},\delta r) + \delta \mu^{T} C^{a}_{g}(g) \delta g,
  \end{aligned}
\end{equation}
where $\delta \mu$ is the Lagrange multiplier of the optimization problem \eqref{DeeNE}. Applying the KKT conditions to \eqref{DeeNE-aug-cost}, one has
\begin{equation}
    \label{DeeNE-KKT}
    \bar{J}^{ne}_{\delta g} = 0, \qquad
    \delta \mu \geq 0.
\end{equation}
where $\bar{J}^{ne}_{\delta g}$ indicates $\partial \bar{J}^{ne}_{N} / \partial \delta g$.

\begin{theorem} [Data-Enabled Neighboring Extremal]
\label{theo1}
Consider the optimization problem \eqref{DeeNE}, the augmented cost function \eqref{DeeNE-aug-cost}, and the KKT conditions \eqref{DeeNE-KKT}. If $\bar{J}_{gg} > 0$, the DeeNE policy
\begin{equation}
  \begin{aligned}
    \label{law}
    &\delta g = K^{*}_1 \delta w_{ini} + K^{*}_2 \delta r,\\
    &K^{*}_1 = -
    \begin{bmatrix}
    I & 0
    \end{bmatrix}
    K^o
    \begin{bmatrix}
    \bar{J}_{gw_{ini}}\\
    {0}
    \end{bmatrix},\\
    &K^{*}_2 = -
    \begin{bmatrix}
    I & 0
    \end{bmatrix}
    K^o
    \begin{bmatrix}
    \bar{J}_{gr}\\
    0
    \end{bmatrix},\\
    &K^{o} =\begin{bmatrix}
    \bar{J}_{gg} & {C^a_{g}}^T\\
    C^a_{g} & 0
    \end{bmatrix}^{-1},
  \end{aligned}
\end{equation}
approximates the perturbed solution for the DeePC \eqref{DeePC} in the presence of initial I/O perturbation $\delta w_{ini}$ and reference perturbation $\delta r$.
\end{theorem}

\begin{proof}
Using \eqref{DeeNE-aug-cost} and the KKT conditions \eqref{DeeNE-KKT}, one has
\begin{equation}
  \begin{aligned}
    \label{J_g}
    &\bar{J}_{gg} \delta g + \bar{J}_{gw_{ini}} \delta w_{ini} + \bar{J}_{gr} \delta r + {C^a_{g}}^T \delta \mu = 0.
      \end{aligned}
\end{equation}
Now, using \eqref{J_g} and the linearized system constraints \eqref{DeeNE}, one has
\begin{equation}
  \begin{aligned}
    \label{g&mu}
    &\begin{bmatrix}
    \bar{J}_{gg} & {C^a_{g}}^T\\
    C^a_{g} & 0
    \end{bmatrix}
    \begin{bmatrix}
    \delta g\\
    \delta \mu
    \end{bmatrix}
    =
    - \begin{bmatrix}
    \bar{J}_{gw_{ini}} \\
    0
    \end{bmatrix}
    \delta w_{ini} -
    \begin{bmatrix}
    \bar{J}_{gr} \\
    0
    \end{bmatrix}
    \delta r,
  \end{aligned}
\end{equation}
which yields
\begin{equation}
  \begin{aligned}
    \label{g&mu 2}
    &\begin{bmatrix}
    \delta g\\
    \delta \mu
    \end{bmatrix}
    = -K^o
    \begin{bmatrix}
    \bar{J}_{gw_{ini}} \\
    0
    \end{bmatrix}
    \delta w_{ini} -K^o
    \begin{bmatrix}
    \bar{J}_{gr} \\
    0
    \end{bmatrix}
    \delta r.
  \end{aligned}
\end{equation}
Thus, the DeeNE policy \eqref{law} is obtained, and the proof is completed.
\end{proof}

\begin{remark} [Singularity]
\label{Singularity}
The assumption of $\bar{J}_{gg} > 0$ and Assumption \ref{ass2} are essential for DeeNE since the first one guarantees the convexity of \eqref{DeeNE}, and both guarantee a non-singular $K^o$ in \eqref{law}. If $C^a_{g}$ was not full row rank, the matrix $K^o$ would be singular, which leads to the failure of the proposed algorithm. This can be addressed using the constraint back-propagation algorithm \cite{ghaemi2008neighboring}.
\end{remark}

\begin{remark}
\label{input u}
Using the control policy \eqref{law}, one can obtain $g^* = g^o + \delta g$, then $u^* = u^o + \delta u$ is obtained using $u = U_F g$. Therefore, one can conclude that $\delta u = K_{ne}^r \delta r + K_{ne}^{ini} \delta w_{ini}$.
\end{remark}

\subsection{DeeNE with  Non-Optimal Nominal Solution}
In the previous part, DeeNE was derived under the assumption of an available nominal DeePC solution; thus, the nominal solution is optimal. In this subsection, we tackle the DeeNE policy for a nominal non-optimal solution so that at each time step, we can use the previous DeeNE solution as the nominal solution in our algorithm. For a nominal non-optimal solution $(g^o,w^o_{ini},r^o)$, we assume that it satisfies the constraints described in \eqref{DeePC g} but may not satisfy the optimality condition $\bar{J}_g (g^o,w_{ini}^o,r^o,\mu^o) = 0$. Under this circumstance, the cost function \eqref{DeeNE-cost} is modified for DeeNE as: 
\begin{equation}
\small
  \begin{aligned}
    \label{DeeNE-cost-large}
  &{J}^{ne}_N(\delta g,\delta w_{ini},\delta r)=\delta^{2} \bar{J}_N(g^o,w^o_{ini},r^o) + \delta \bar{J}_N(g^o,w^o_{ini},r^o)= \\
  &\frac{1}{2} 
  \begin{bmatrix}
  \delta g\\ \delta w_{ini}\\ \delta r
  \end{bmatrix}^{T}
  \begin{bmatrix}
  \bar{J}_{gg} & \bar{J}_{gw_{ini}} & \bar{J}_{gr}\\
  \bar{J}_{w_{ini}g} & \bar{J}_{w_{ini}w_{ini}} & \bar{J}_{w_{ini}r}\\
  \bar{J}_{rg} & \bar{J}_{rw_{ini}} & \bar{J}_{rr}
  \end{bmatrix}
  \begin{bmatrix}
  \delta g\\ \delta w_{ini}\\ \delta r
  \end{bmatrix} + \bar{J}^T_g \delta g.
    \end{aligned}
\end{equation}

Considering the optimal control problem \eqref{DeeNE} and the cost function \eqref{DeeNE-cost-large}, the augmented cost function is modified as:
\begin{equation}
  \begin{aligned}
    \label{DeeNE-aug-cost-large}
    &\bar{J}^{ne}_N(\delta g,\delta w_{ini},\delta r,\delta \mu) = \delta^{2} \bar{J}_N(g^o,w^o_{ini},r^o) + \delta \mu^{T} C^{a}_{g}(g) \delta g \\
    & \hspace{35 mm} + \bar{J}^T_g(g^o,w^o_{ini},r^o) \delta g.
  \end{aligned}
\end{equation}

Now, the following theorem is presented to modify the DeeNE policy for the nominal non-optimal solutions.

\begin{theorem} [Modified Data-Enabled Neighboring Extremal]
\label{theo2}
Consider the optimization problem \eqref{DeeNE}, the KKT conditions \eqref{DeeNE-KKT}, and the augmented cost function \eqref{DeeNE-aug-cost-large}. If $\bar{J}_{gg} > 0$, then the DeeNE policy is modified for a nominal non-optimal solution as
\begin{equation}
  \begin{aligned}
    \label{law2}
    &\delta g = K^{*}_1 \delta w_{ini} + K^{*}_2 \delta r + K^{*}_3 \begin{bmatrix}
    \bar{J}_{g}\\
    {0}
    \end{bmatrix},\\
    &K^{*}_3 = -
    \begin{bmatrix}
    I & 0
    \end{bmatrix}
    K^o,
  \end{aligned}
\end{equation}
where the gain matrices $K^{*}_1$, $K^{*}_2$, and $K^o$ are defined in \eqref{law}.
\end{theorem}

\begin{proof}
Using the KKT conditions \eqref{DeeNE-KKT} and the modified augmented cost function \eqref{DeeNE-aug-cost-large}, one has
\begin{equation}
  \begin{aligned}
    \label{M_J_g}
    &\bar{J}_{gg} \delta g + \bar{J}_{gw_{ini}} \delta w_{ini} + \bar{J}_{gr} \delta r + {C^a_{g}}^T \delta \mu + \bar{J}_{g} = 0.
      \end{aligned}
\end{equation}
Now, using \eqref{M_J_g} and the linearized system constraints \eqref{DeeNE}, one has
\begin{equation}
  \begin{aligned}
    \label{m_g&mu}
    &\begin{bmatrix}
    \bar{J}_{gg} & {C^a_{g}}^T\\
    C^a_{g} & 0
    \end{bmatrix}
    \begin{bmatrix}
    \delta g\\
    \delta \mu
    \end{bmatrix}
    =
    - \begin{bmatrix}
    \bar{J}_{gw_{ini}} \\
    0
    \end{bmatrix}
    \delta w_{ini} -
    \begin{bmatrix}
    \bar{J}_{gr} \\
    0
    \end{bmatrix}
    \delta r - \begin{bmatrix}
    \bar{J}_{g}\\
    {0}
    \end{bmatrix},
  \end{aligned}
\end{equation}
which yields
\begin{equation}
  \begin{aligned}
    \label{m_g&mu 2}
    &\begin{bmatrix}
    \delta g\\
    \delta \mu
    \end{bmatrix}
    = -K^o
    \begin{bmatrix}
    \bar{J}_{gw_{ini}} \\
    0
    \end{bmatrix}
    \delta w_{ini} -K^o
    \begin{bmatrix}
    \bar{J}_{gr} \\
    0
    \end{bmatrix}
    \delta r - K^o
    \begin{bmatrix}
    \bar{J}_{g}\\
    {0}
    \end{bmatrix}.
  \end{aligned}
\end{equation}
Thus, the modified DeeNE policy \eqref{law2} is obtained, and the proof is completed.
\end{proof}

\begin{remark} [Quadratic Cost]
\label{Quad J}
One can consider a quadratic cost function $J_N(\mathbf{y},\mathbf{u},\mathbf{\sigma_y},\mathbf{\sigma_u},\mathbf{g},\mathbf{r})$ as:
\begin{equation}
  \begin{aligned}
    \label{cost deepc}
    & J_N(\mathbf{y},\mathbf{u},\mathbf{\sigma_y},\mathbf{\sigma_u},\mathbf{g},\mathbf{r}) = {\|y-r\|}_Q^2 + {\|u\|}_R^2 + \lambda_y {\|\sigma_y\|}_2^2 \\
    & \hspace{35 mm} + \lambda_u {\|\sigma_u\|}_2^2 + \lambda_g {\|g\|}_2^2,
  \end{aligned}
\end{equation}
where the positive semi-definite matrix $Q \in \mathbb{R}^{pN \times pN}$ and the positive definite matrix $R \in \mathbb{R}^{mN \times mN}$ are weighting matrices, and the positive parameters $\lambda_y, \lambda_u, \lambda_g \in \mathbb{R}$ are regularization weights. For the quadratic cost function \eqref{cost deepc}, the DeePC is a quadratic programming (QP) problem on the decision variable $g$, which requires an iterative QP solver. However, using DeeNE, one can have a computationally efficient solution to this optimization problem without requiring an iterative solver. For DeeNE framework, one has
\begin{equation}
  \begin{aligned}
      \label{Quad DeeNE}
    & \bar{J}_{g} = 2 ((Y_F g - r)^T Q Y_F + (U_F g)^T R U_F\\
    & \hspace{5 mm} + \lambda_y (Y_P g - y_{ini})^T Y_P + \lambda_u (U_P g - u_{ini})^T U_P + \lambda_g g^T),\\ 
    & \bar{J}_{gg} = 2 (Y^T_F Q Y_F + U^T_F R U_F + \lambda_y Y^T_P Y_P + \lambda_u U^T_P U_P + \lambda_g),\\
    & \bar{J}_{gw_{ini}} = - 2 (\lambda_y Y^T_P + \lambda_u U^T_P),\\
    & \bar{J}_{gr} = - 2 Y^T_F Q,\\
\end{aligned}
\end{equation}
where the requirement $\bar{J}_{gg}>0$ is satisfied.
\end{remark}

\subsection{DeeNE Implementation}
In this subsection, we present how DeeNE is implemented for efficient control of autonomous systems, which is summarized in Algorithm 1. Given the pre-collected Hankel matrices, an initial I/O trajectory $w_{ini}$ and the reference trajectory $r$, it uses DeePC~\eqref{DeePC g} for the time step $k = T_{ini}$ to generate an N-length solution. Then, the DeeNE framework applies the first $s$ control inputs $u^*(k:k+s)$, where $s\in[0:1:N]$ is a hyper-parameter that can be tuned. For the remaining $T_c-s$ steps, where $T_c$ is the simulation/experiment time, the optimal adaptation laws developed above are used to correct the control action through $\delta g$, based on the deviations $\delta w_{ini}$ and $\delta r$. Note that at the first correction step, i.e., $k = T_{ini} + s + 1$, Theorem 1 and Theorem 2 represent the same control policy having $\bar{J}_g (g^o,w_{ini}^o,r^o,\mu^o) = 0$ since we have a nominal DeePC solution. However, Theorem 2 needs to be used for the remaining steps to consider the nominal solution as the DeeNE solution from the previous step. This process then continues for the next steps. In particular, only one DeePC optimization is performed throughout the simulation/experiment time $T_c$, significantly reducing computational complexity compared to the traditional DeePC which requires solving a new optimization problem ~\eqref{DeePC g} with updated $w_{ini}$ and $r$ at each step. 

\begin{figure}
\removelatexerror
\scalebox{0.85}{
\begin{algorithm*}[H]
\SetAlFnt{\small}
    \SetKwInOut{Parameter}{Parameter}
    \SetKwInOut{Input}{Input}
    \SetKwInOut{Output}{Output}
\caption{Modified DeeNE}
\label{DeeNE Algorithm}
\SetAlgoLined
\Parameter{$U_P$, $Y_P$, $U_F$, $Y_F$, $C$, $Q$, $R$, $\lambda_y$, $\lambda_u$, $\lambda_g$.}
\Input{$w_{ini}(0:T_{ini}-1)$, $r(0:N-1)$.}
\Output{$\mathbf{u}(0:T_c)$, $\mathbf{y}(0:T_c)$.}
\vspace{0.2em}
\hrule
\vspace{0.2em}
\For{$k = T_{ini}$ : s : $T_c$}{
     \If{$k == T_{ini}$}{
        Compute $g^*$ using \eqref{DeePC g}\;
        $u^* = U_F g^*$\;
     }
     \Else{
        Calculate $\mu^{o}$ using \eqref{Lagrange multipliers}\;
        Calculate $K^*_1$, $K^*_2$, $K^*_3$ using \eqref{law} and \eqref{law2}\;
        $\delta w_{ini} = w_{ini} - w_{ini}^o$\; 
        $\delta r = r - r^o$\;
        Calculate $\delta g$ using \eqref{law2}\;
        $g^* = g^o + \delta g$\; 
        $u^* = U_F g^*$\;
     }
     Apply $u(k:k+s) = u^*(0:s)$, $0 \leq s \leq N-1$\;
     Measure $y(k:k+s)$ from the system\;
     $g^o = g^*$\; 
     $w_{ini}^o = w_{ini}$\; 
     $r^o = r$\;
     Update $w_{ini} = w(k+s-Tini+1:k+s)$\;
     Update $r(0:N-1)$\;
}
\end{algorithm*}}
\end{figure}

\begin{figure}[!h]
     \centering
     \includegraphics[width=0.99\linewidth]{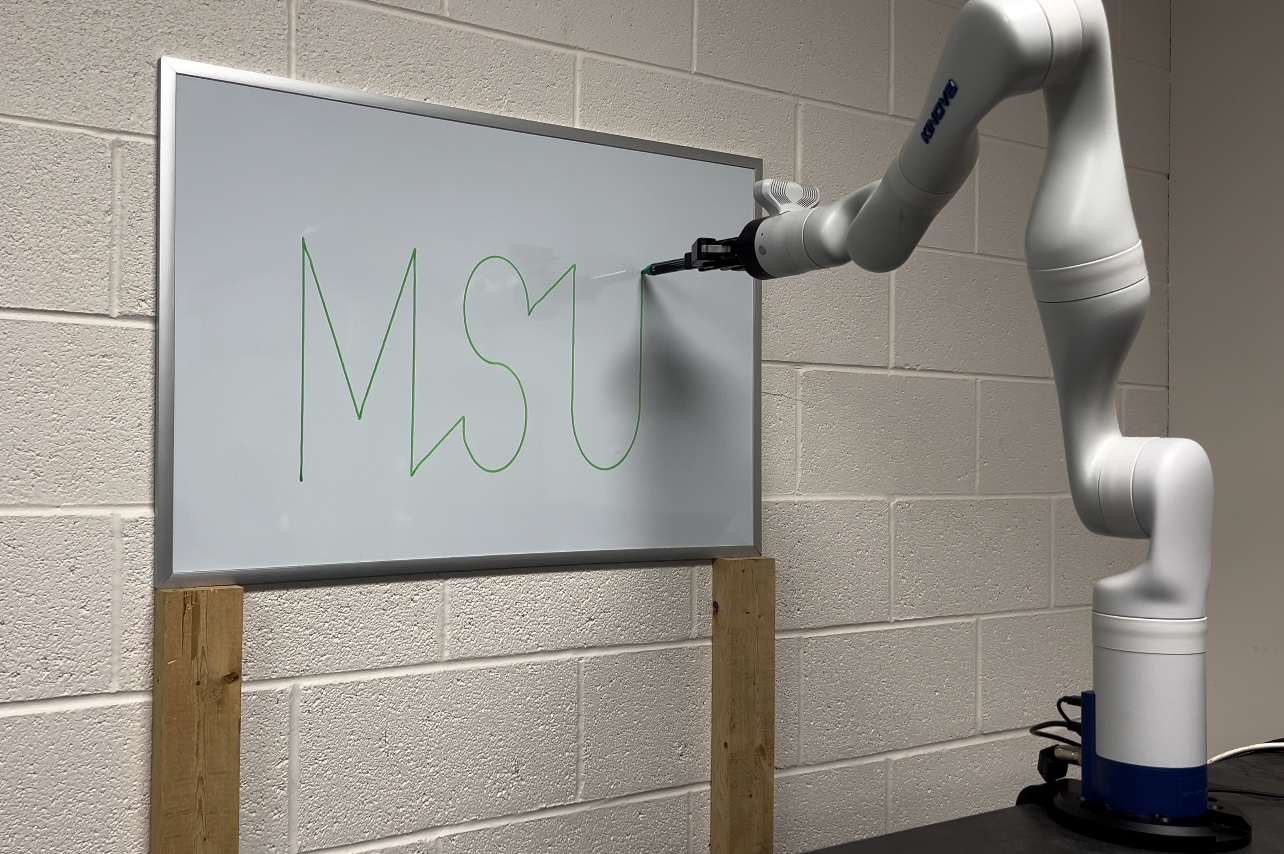}
    \caption{Drawing MSU by 7-DoF robotic arm.}
     \label{Robotic_Arms}
 \end{figure}

\section{Case Study on Robotic Arm Motion Control with Constraints}
\label{Sec5}
In this section, we present a case study of the developed DeeNE framework on the reference tracking of a  7-DoF robotic manipulator, KINOVA Gen3. As shown in Fig. \ref{Robotic_Arms}, the goal is to control the Robotic Arm to draw certain patterns such as ``MSU''. Note that forward and inverse kinematics of the Robotic Arm is needed to enable a model-based control for the desired pose tracking. However, deriving the forward/inverse kinematics of the 7-DoF robotic arm requires domain expertise, and extensive parameter calibrations may be needed to realize a high-precision model-based control. 
As the robot's DoF increases, the complexity of deriving accurate forward and inverse kinematics relationships grows significantly. Consider a soft robotic arm, the model-based methods become increasingly cumbersome and computationally expensive as they require managing numerous joint interactions and potential singularities. This complexity can hinder the real-time performance of the robot in dynamic environments. To address these challenges, data-driven approaches such as DeePC or DeeNE is advantageous as they offer greater flexibility and adaptability in handling higher-DoF robots.

\subsection{System Specification and DeePC/DeeNE Formulation}
According to the manufacturer specifications and safety considerations, we consider the minimum and maximum values for the angular velocities (control inputs) of all seven joints as $[-\pi/6, \pi/6]
rad/s$. Limitations on the Cartesian position coverage are set as $[-0.9, 0.9] m$ in the 3D space. The inputs $u \in \mathbb{R}^7$ include the seven joint angular velocities while the output  $y \in \mathbb{R}^6$ comprises the pose of the robot that includes the 3D position of the end-effector, i.e., $d = [d_x, d_y, d_z]^T$, and the orientation of the end-effector in X-Y-Z Euler angles, i.e., $\theta = [\theta_x, \theta_y, \theta_z]^T$. To avoid a discontinuous behavior in the orientation part, we convert the 3D orientation to 4D orientation using Quaternions \cite{ozgur2016kinematic}. 

The protocol of data collection is as follows. We collected data from the 7-DoF robotic arm for $50$ trajectories with $T_i = 100$ data points on each trajectory and the sampling time $T_s = 0.1 s$. It is worth noting that since we are generating the Hankel matrix using multiple signal trajectories, called mosaic-Hankel matrix (a Hankel matrix with discontinuous signal trajectories), the number of data points on each trajectory must be greater than the depth of the Hankel matrix, i.e. $T_i > T_{ini} + N$ \cite{vahidi2024online}. For each trajectory, the initial joint angles and the inputs are chosen randomly according to a uniform probability distribution. Due to the setup condition in the lab (desk structure, wall position, etc.), we had to stop the robot if it was close to hit an object, ignore that trajectory, and continue the data collection with another initial position and/or input values. 

Details of DeePC are as follows. The reference trajectory $r(k) \in \mathbb{R}^7$ represents the desired values for the pose of the robot. According to the quadratic cost function, the matrices $Q = 5 \times 10^4 \times I_{pN}$ and $R = 1 \times 10^2 \times I_{mN}$ are considered to penalize the tracking error and control input amplitude, respectively. The slack variables $\lambda_y, \lambda_u = 5 \times 10^5$ are used to ensure the feasibility of the optimal control problem. The regularization parameter $\lambda_g = 5 \times 10^2$ avoids the overfitting issue due to the collected noisy data. Finally, the initial trajectory and the prediction lengths are $T_{ini} = 35$ and $N = 20$, respectively. Since we have $u \in \mathbb{R}^7$ and $y \in \mathbb{R}^7$, the dimension of the mosaic-Hankel matrix is $\mathbb{H}(u^d,y^d) \in \mathbb{R}^{770 \times 2300}$ causing high computational time for applying a real-time DeePC on the 7-DoF robotic arm. For DeePC, we use the DeePC policy \eqref{DeePC}, apply the first $s$ optimal control input $u(k:k+s)$ to the 7-DoF robotic arm, measure the pose of the robot, and update the initial trajectory $w_{ini}$ and the reference trajectory $r$ for the next step (See Algorithm 2 in \cite{coulson2019data}). For an initial pose of the robotic arm, we generate the first initial trajectory $(u_{ini},y_{ini})$ using random control inputs, i.e. $(u(0:34),y(0:34))$. For a tracking performance index, we use Root Mean Square Error (RMSE) between the desired reference trajectory and the pose of the robotic arm over the entire trajectory.

\subsection{Simulation Results}
In this subsection, we conduct simulation studies to compare the performance of the proposed DeeNE framework with DeePC by evaluating the tracking performance and computational time for different open-loop control scenarios, i.e., under various lengths $s$. For this part, we use the forward kinematics model of the 7-DoF robotic arm (shared by the arm manufacturer) to evaluate the performance of the control schemes in simulations. The reference trajectory $r(k)$ is chosen as a sinusoidal trajectory with $300$ data points for the pose of the end-effector. For the desired reference trajectory $r(k)$, DeePC and DeeNE both seek to accomplish the reference tracking task. For this case, we use DeeNE policy \eqref{law2} to correct the open-loop DeePC solution at each step so as to reduce the computational time. The tracking performance and the computational time are examined for DeePC and DeeNE frameworks under different open-loop control lengths $s$. Fig. \ref{Inputs_Sim} compares the control input for the open-loop control scenario $s=20$, which illustrates how DeeNE corrects the open-loop DeePC control sequences. Figs. \ref{Outputs3D_Sim} and \ref{Outputs_Sim} show the position and orientation tracking performance, respectively, where one can see that the open-loop DeePC does not track the reference trajectory for $s = 20$; however, DeeNE tracks the reference very well, thanks to the efficient adaptation scheme detailed in Algorithm 1. This shows the necessity of the online correction. Furthermore, Table I summarizes the tracking performance and the computational time indices for both control algorithms under $s = 0$, $s = 10$, and $s = 20$. One can see that both controllers perform similarly on the tracking performance for $s = 0$, corresponding to the case that optimization is performed every step for DeePC, but DeeNE provides significantly lower computational time. However, as we increase $s$, the length of open-loop control applications, the performance of DeePC goes down since no adaptation is utilized and it predicts the behavior of the system using the last available initial and reference trajectories. On the other hand, DeeNE takes feedback from the system and updates the initial and reference trajectories at each time step, which corrects the DeePC predictions. Consequently, one can see that DeeNE enables both high-precision tracking performance and faster motion speed for the 7-DoF robotic arm.

 \begin{figure}[!h]
     \centering
     \includegraphics[width=0.99\linewidth]{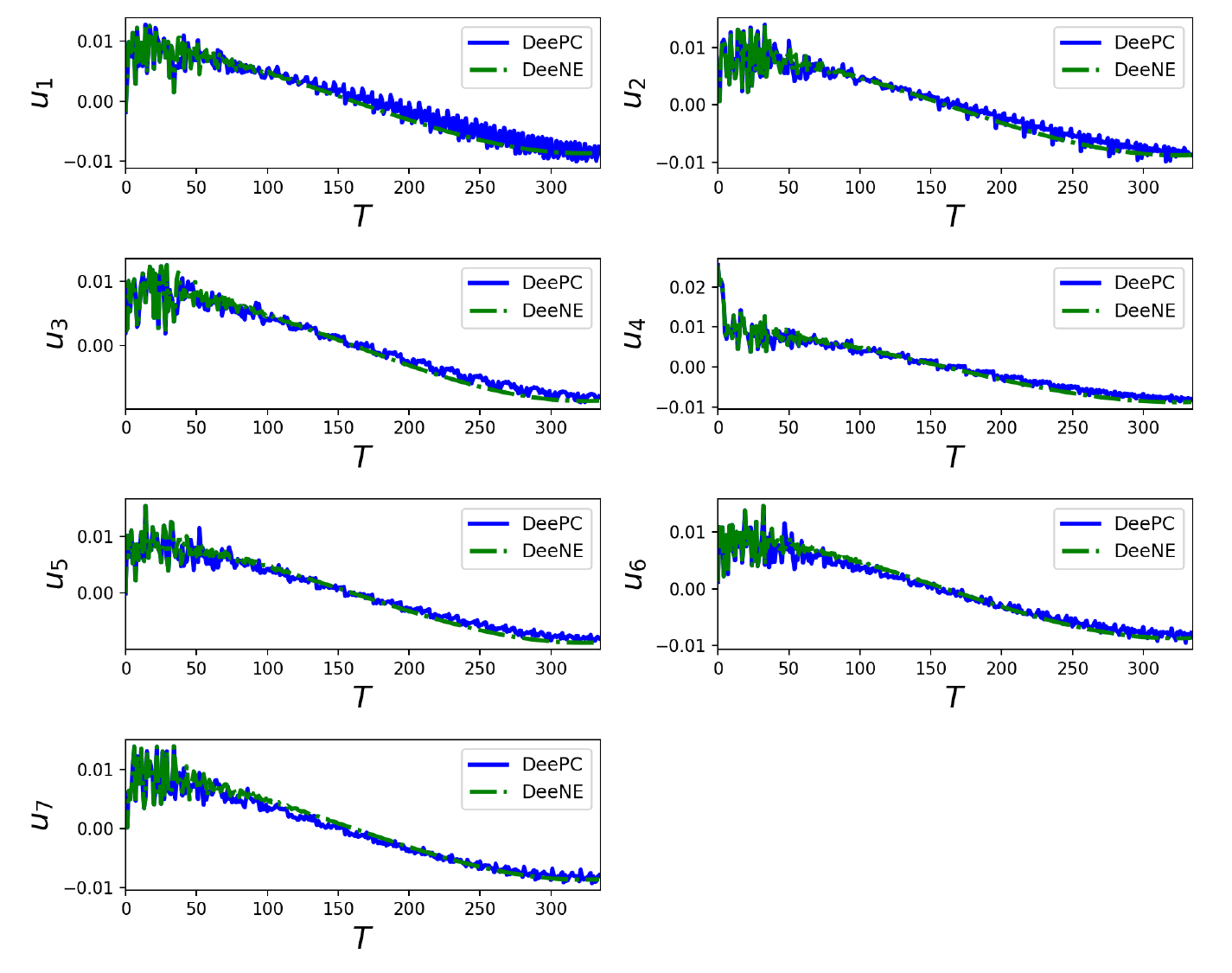}
    \caption{Control input for 7-DoF Robotic Arm (Simulation).}
     \label{Inputs_Sim}
 \end{figure}
 
 \begin{figure}[!h]
     \centering
     \includegraphics[width=0.99\linewidth]{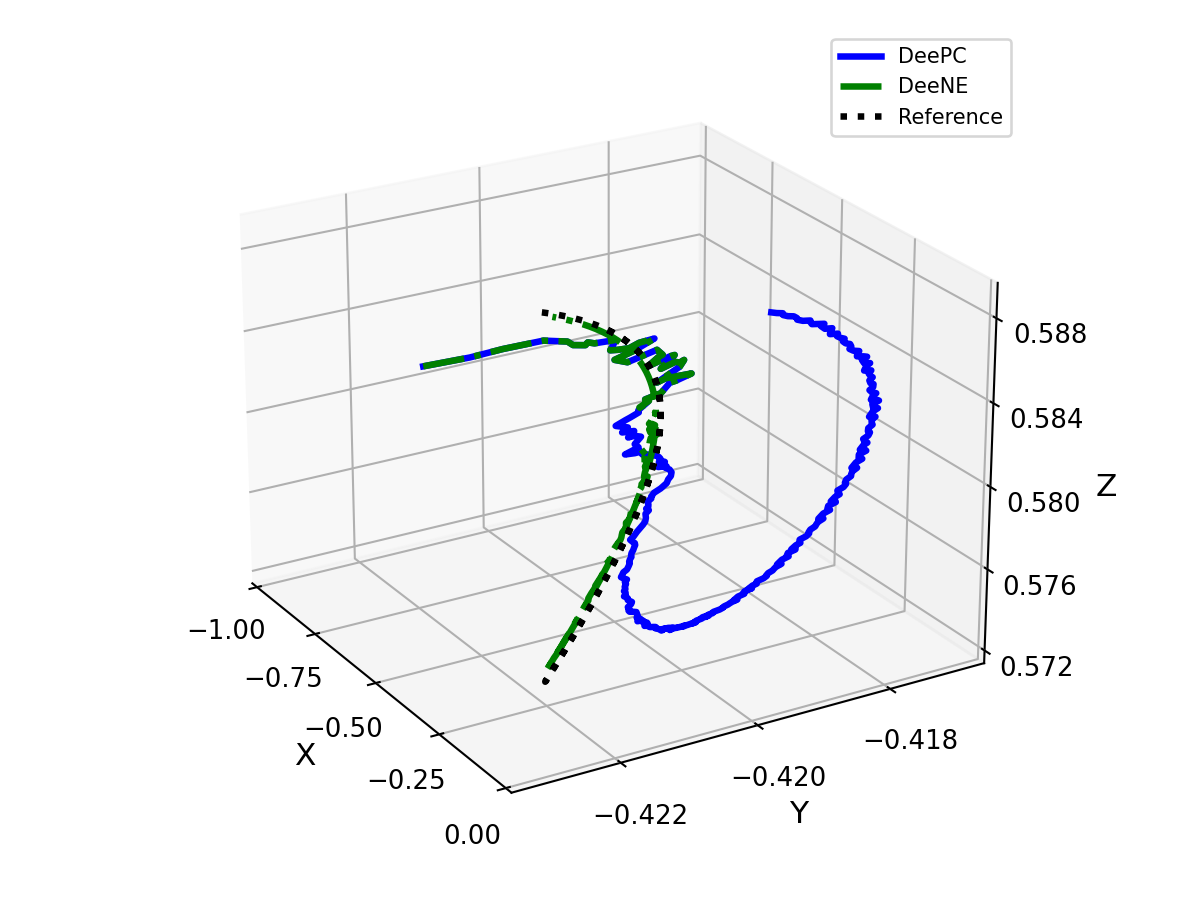}
    \caption{Position tracking for 7-DoF Robotic Arm (Simulation).}
     \label{Outputs3D_Sim}
 \end{figure}

\begin{figure}[!h]
     \centering
     \includegraphics[width=0.99\linewidth]{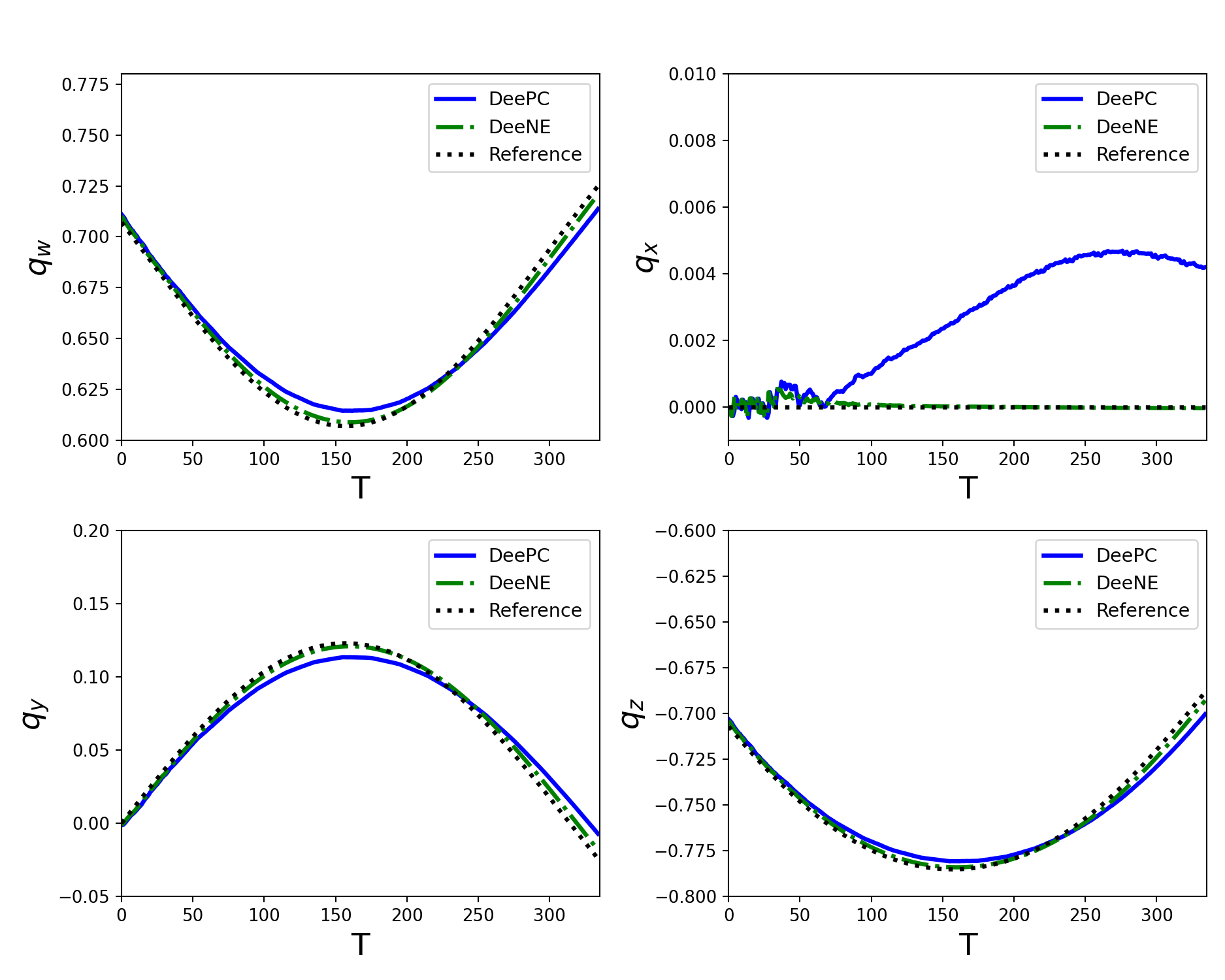}
    \caption{Orientation tracking for 7-DoF Robotic Arm (Simulation).}
     \label{Outputs_Sim}
 \end{figure}
 
\vspace{-5 pt}
\begin{table}[!ht]
\centering
 \caption{Comparison of Performance and Computational Time for DeePC and DeeNE with different open-loop control scenarios}
\begin{tabular}{ |p{2.3cm}|p{2.3cm}|p{2.3cm}|  }
\hline
\hline
Controller & RMSE & Time (per loop) \\
\hline
DeePC (s = 0)  & $0.23 \hspace{1 mm} cm$ & $20.02 \hspace{1 mm} ms$ \\\hline
DeePC (s = 10) & $0.48 \hspace{1 mm} cm$ & $2.04 \hspace{1 mm} ms$ \\\hline
DeePC (s = 20) & $0.51 \hspace{1 mm} cm$ & $1.14 \hspace{1 mm} ms$ \\\hline
DeeNE (s = 0)  & $0.24 \hspace{1 mm} cm$ & $3.03 \hspace{1 mm} ms$ \\\hline
DeeNE (s = 10) & $0.27 \hspace{1 mm} cm$ & $0.39 \hspace{1 mm} ms$ \\\hline
DeeNE (s = 20) & $0.32 \hspace{1 mm} cm$ & $0.25 \hspace{1 mm} ms$ \\\hline
\hline
\end{tabular}
\end{table}

\subsection{Experimental Results}
In this part, we apply both DeePC and DeeNE on the real 7-DoF robotic arm for a closed-loop control scenario (i.e., $s = 0$) to ensure the safety and stability of the robot under DeePC. We consider the task of using the robot to draw on a board, with the target reference chosen as ``MSU'', the abbreviation of Michigan State University. For the orientation of the end effector, we consider the $0.5$ degree as the desired orientation to show the performance of the controllers for orientation tracking and also to avoid the rotation of the marker. For the desired reference trajectory $r(k)$, DeePC and DeeNE must simultaneously accomplish reference tracking and setpoint control tasks for the position and orientation of the end-effector, respectively. Similar to the simulations, we use DeeNE policy \eqref{law2} to avoid solving the DeePC problem at each time step and thus reduce the computational time,  providing a much faster motion control sampling rate for the robot. Fig. \ref{Inputs} compares the control inputs generated by both control algorithms, which illustrates the effectiveness of DeeNE on the approximation of the DeePC policy. Figs. \ref{Outputs3D} and \ref{Outputs} show the position and orientation tracking performance, respectively, where one can see that both controllers track the reference trajectory very well. Table II lists the tracking performance and the computational time indices for both control algorithms with $s=0$, showing similar tracking performance but that DeeNE has much lower computational time. It is worth noting that since the computational time of DeePC, i.e., $0.2 s$, is higher than the sampling time of the robot hardware, i.e., $0.1 s$, the robotic arm receives the response of DeePC for the first $0.1 s$, waits for the second $0.1 s$ until receiving the updated response of DeePC, and then the process is repeated using the updated control input, leading to the slow and discontinuous motion of the robot. It can be seen that  DeeNE effectively achieves both high-precision tracking performance and faster motion speed for the 7-DoF robotic arm. 

 \begin{figure}[!h]
     \centering
     \includegraphics[width=0.99\linewidth]{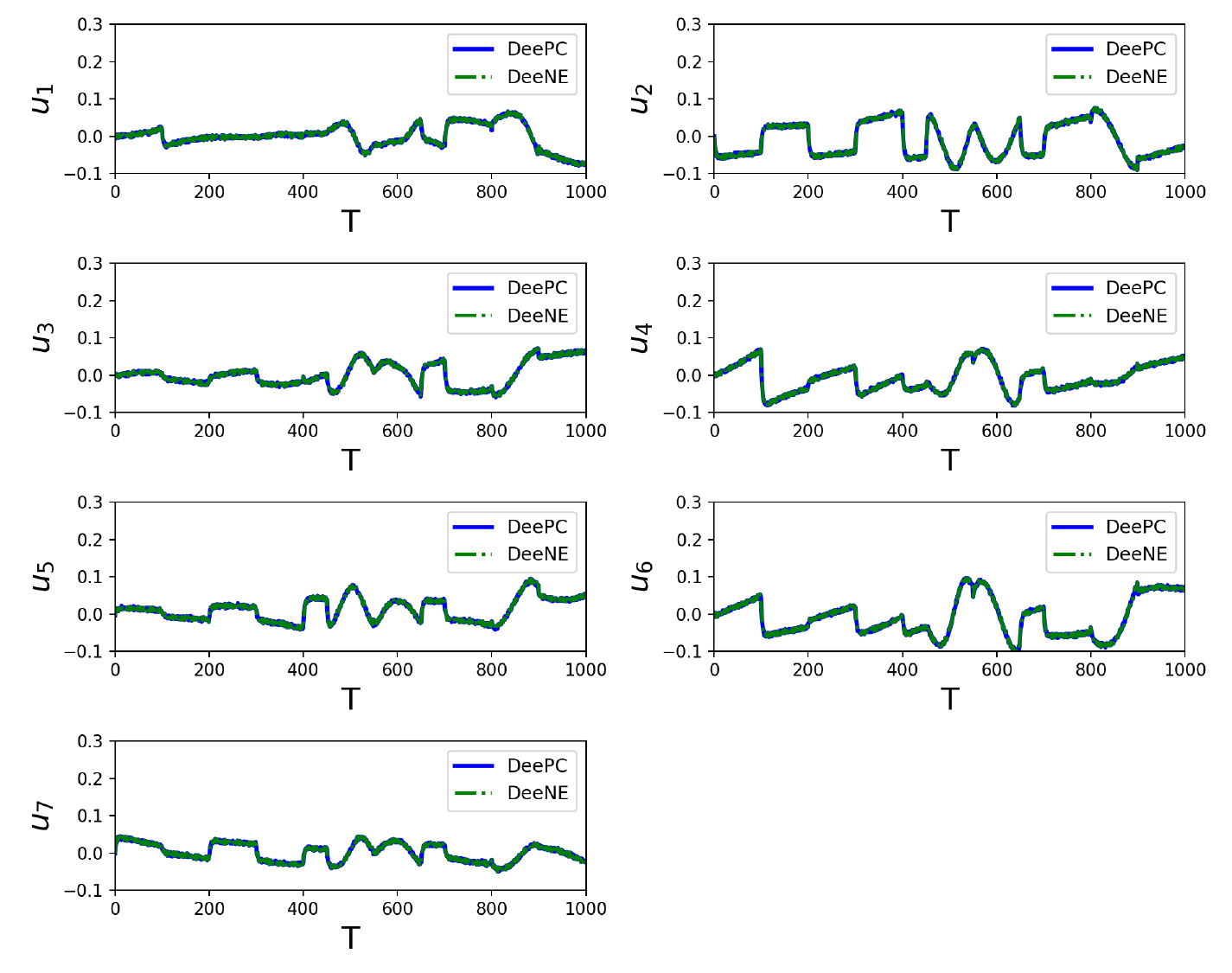}
    \caption{Control input for 7-DoF Robotic Arm (Experiment).}
     \label{Inputs}
 \end{figure}
 
 \begin{figure}[!h]
     \centering
     \includegraphics[width=0.99\linewidth]{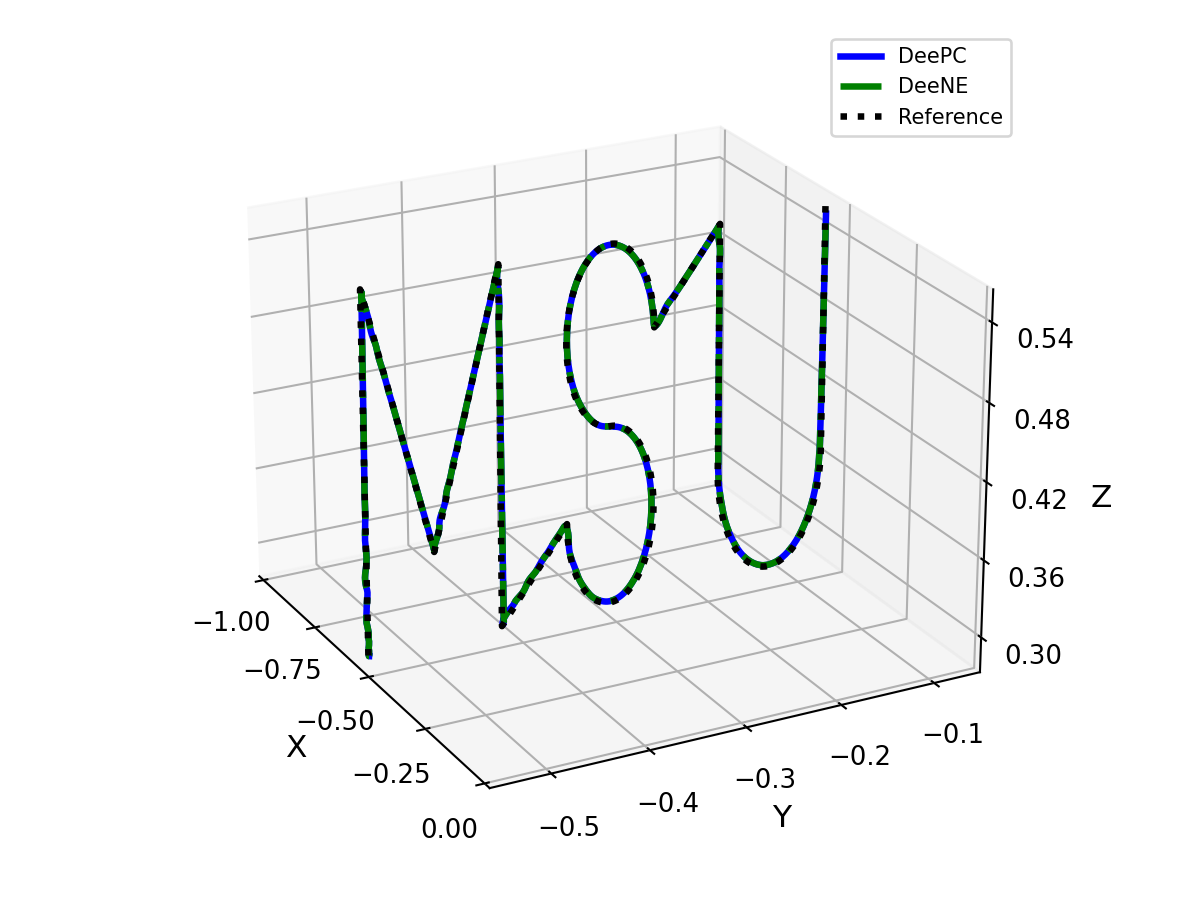}
    \caption{Position tracking for 7-DoF Robotic Arm (Experiment).}
     \label{Outputs3D}
 \end{figure}

\begin{figure}[!h]
     \centering
     \includegraphics[width=0.99\linewidth]{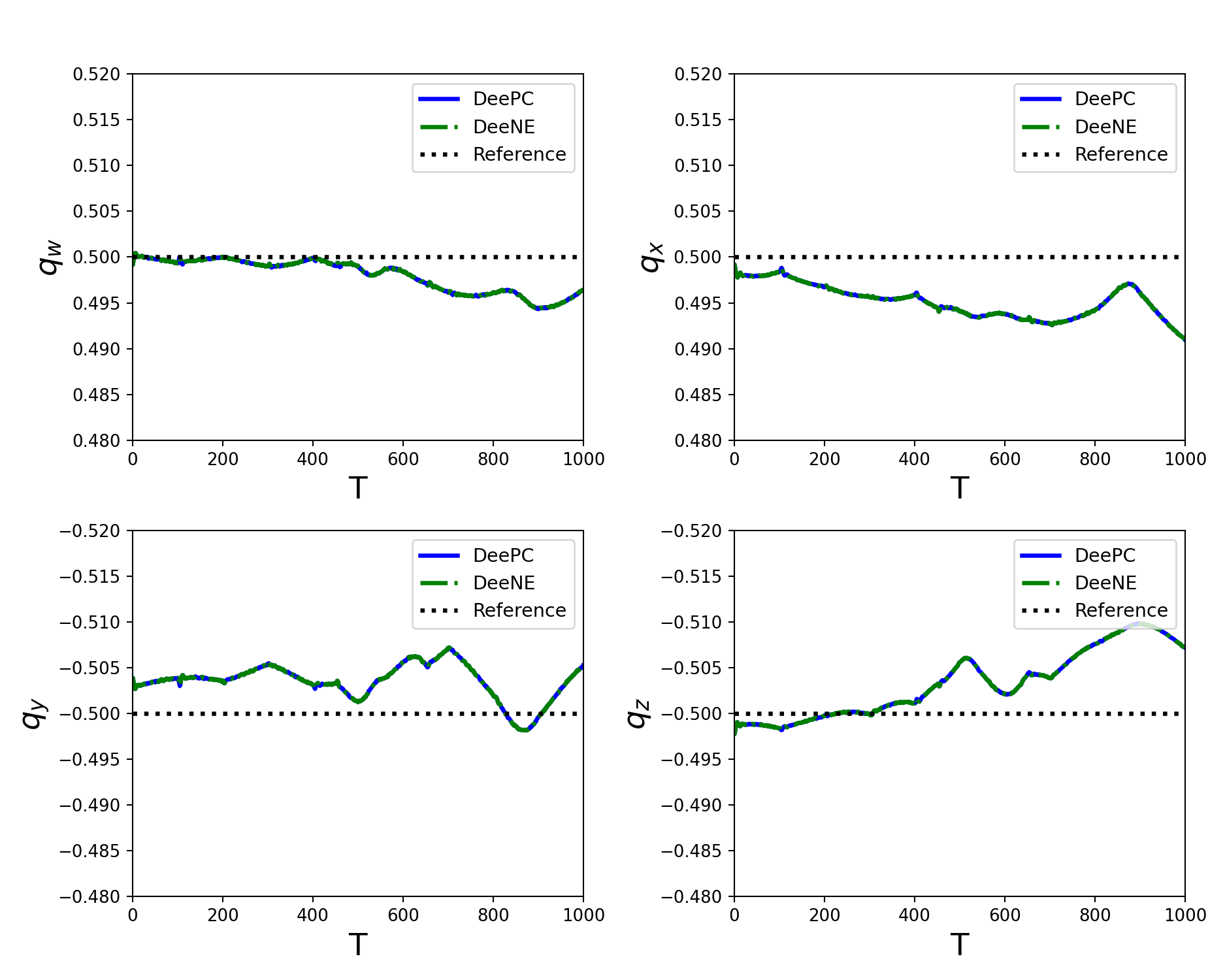}
    \caption{Orientation tracking for 7-DoF Robotic Arm with a y-axis scale increment of 0.005 (Experiment).}
     \label{Outputs}
 \end{figure}

\vspace{-5 pt}
\begin{table}[!ht]
\centering
 \caption{Comparison of Performance and Computational Time for DeePC and DeeNE}
\begin{tabular}{ |p{2.3cm}|p{2.3cm}|p{2.3cm}|  }
\hline
\hline
Controller & RMSE & Time (per loop) \\
\hline
DeePC & $1.42 \hspace{1 mm} cm$ & $20.05 \hspace{1 mm} ms$ \\\hline
DeeNE & $1.43 \hspace{1 mm} cm$ & $3.07 \hspace{1 mm} ms$ \\\hline
\hline
\end{tabular}
\end{table}

We next verify the performance of the control algorithms under safety constraints, where the robotic arm must avoid unsafe regions such as dynamic obstacles. In this setting, the robot must track the same reference trajectory ``MSU''; however, we consider an unsafe region illustrated as a red box on the top part of ``s''. This setup addresses, for example, cases where the reference trajectory is obtained offline using path planning or by a remote operator, but the controller must avoid unsafe regions due to dynamic obstacles. In this case, we do not use online path planning to update the reference trajectory and avoid the dynamic obstacles. As shown in Figs. \ref{Inputs_Safe}-\ref{Outputs_Safe}, both DeePC and DeeNE can satisfy the safety constraints and track the reference trajectory well in its best capacity. However, similar to the previous tasks and as shown in Table III, the computational time of DeePC is higher than the sampling time, making it impractical for real-time implementations. On the other hand, DeeNE renders a much faster computation time while achieving very similar performance, demonstrating its efficacy. 
The demo video of the experiments can be found at the following link \url{https://www.youtube.com/watch?v=BlKTUgkAMVo}.    

 \begin{figure}[!h]
     \centering
     \includegraphics[width=0.99\linewidth]{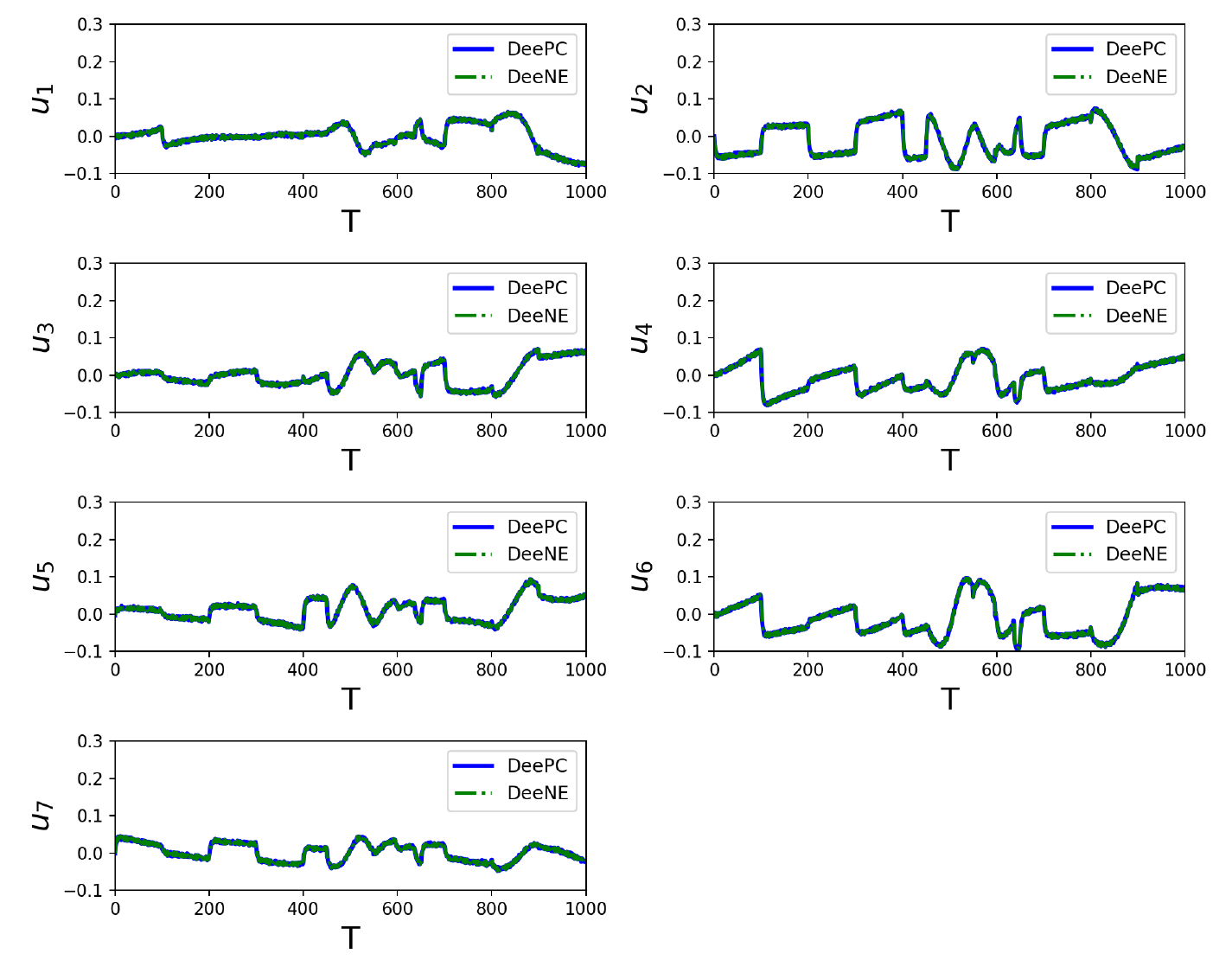}
    \caption{Safe control input for 7-DoF Robotic Arm (Experiment).}
     \label{Inputs_Safe}
 \end{figure}

\begin{figure}[!h]
     \centering
     \includegraphics[width=0.99\linewidth]{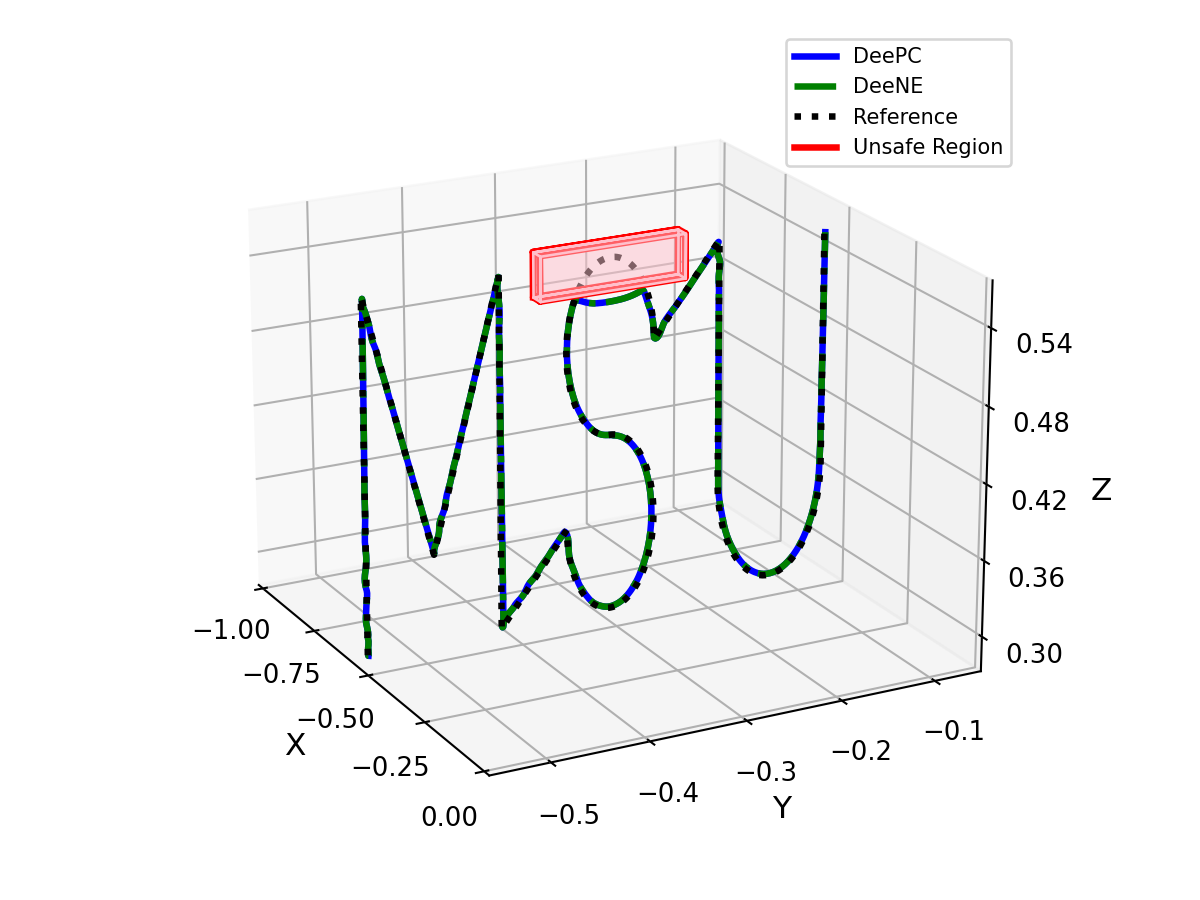}
    \caption{Safe position tracking for 7-DoF Robotic Arm (Experiment).}
     \label{Outputs3D_Safe}
 \end{figure}
 
 \begin{figure}[!h]
     \centering
     \includegraphics[width=0.99\linewidth]{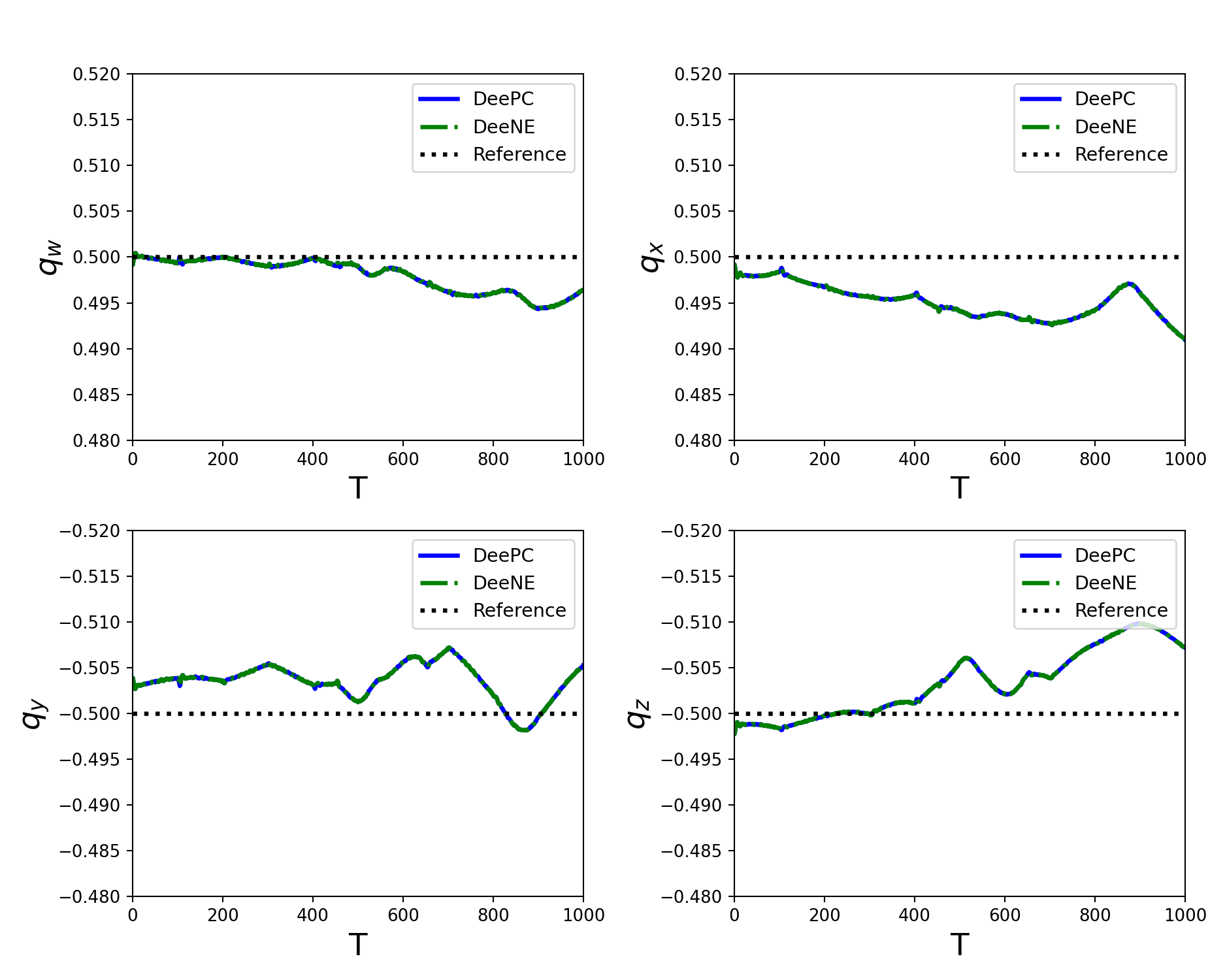}
    \caption{Safe orientation tracking for 7-DoF Robotic Arm with a y-axis scale increment of 0.005 (Experiment).}
     \label{Outputs_Safe}
 \end{figure}
 
\vspace{-5 pt}
\begin{table}[!ht]
\centering
 \caption{Comparison of Performance and Computational Cost for DeePC and DeeNE with safety guarantees}
\begin{tabular}{ |p{2.3cm}|p{2.3cm}|p{2.3cm}|  }
\hline
\hline
Controller & RMSE & Time (per loop) \\
\hline
DeePC & $1.48 \hspace{1 mm} cm$ & $20.08 \hspace{1 mm} ms$ \\\hline
DeeNE & $1.49 \hspace{1 mm} cm$ & $3.09 \hspace{1 mm} ms$ \\\hline
\hline
\end{tabular}
\end{table}

\section{Conclusion}
\label{Sec6}
This paper presented a novel approach to improving the computational efficiency of DeePC for trajectory tracking tasks. Specifically, DeeNE was developed to optimally correct/approximate the DeePC policy in the presence of I/O and reference trajectory perturbations. The developed DeeNE was based on the second-order variation of the original DeePC problem such that its computational load grows linearly with the optimization horizon. This control approach alleviates the online computational burden and extends the applicability of DeePC in many real-time systems. Simulation and experimental verifications on the 7-DoF robotic arm demonstrated the DeeNE's substantial computation saving over the DeePC, while retaining similar performance. Future work will involve the integration of dimension reduction techniques into DeeNE to further improve the computation efficiency.

\bibliographystyle{ieeetr}
\bibliography{References.bib}

\begin{thebibliography}{10}

\bibitem{liu2022mpc}
X.~Liu, W.~Wang, X.~Li, F.~Liu, Z.~He, Y.~Yao, H.~Ruan, and T.~Zhang, ``Mpc-based high-speed trajectory tracking for 4wis robot,'' {\em ISA transactions}, vol.~123, pp.~413--424, 2022.

\bibitem{li2024identification}
Y.~Li, Z.~Xu, X.~Yang, Z.~Zhao, L.~Zhuang, J.~Zhao, and H.~Liu, ``Identification and high-precision trajectory tracking control for space robotic manipulator,'' {\em Acta Astronautica}, vol.~214, pp.~484--495, 2024.

\bibitem{zhu2018high}
Y.~Zhu, J.~Qiao, Y.~Zhang, and L.~Guo, ``High-precision trajectory tracking control for space manipulator with neutral uncertainty and deadzone nonlinearity,'' {\em IEEE Transactions on Control Systems Technology}, vol.~27, no.~5, pp.~2254--2262, 2018.

\bibitem{salzmann2023real}
T.~Salzmann, E.~Kaufmann, J.~Arrizabalaga, M.~Pavone, D.~Scaramuzza, and M.~Ryll, ``Real-time neural mpc: Deep learning model predictive control for quadrotors and agile robotic platforms,'' {\em IEEE Robotics and Automation Letters}, vol.~8, no.~4, pp.~2397--2404, 2023.

\bibitem{williams2017information}
G.~Williams, N.~Wagener, B.~Goldfain, P.~Drews, J.~M. Rehg, B.~Boots, and E.~A. Theodorou, ``Information theoretic mpc for model-based reinforcement learning,'' in {\em 2017 IEEE international conference on robotics and automation (ICRA)}, pp.~1714--1721, IEEE, 2017.

\bibitem{vahidi2023unified}
A.~Vahidi-Moghaddam, K.~Chen, K.~Zhang, Z.~Li, Y.~Wang, and K.~Wu, ``A unified framework for online data-driven predictive control with robust safety guarantees,'' {\em arXiv preprint arXiv:2306.17270}, 2023.

\bibitem{zhang2023dimension}
K.~Zhang, Y.~Zheng, C.~Shang, and Z.~Li, ``Dimension reduction for efficient data-enabled predictive control,'' {\em IEEE Control Systems Letters}, vol.~7, pp.~3277--3282, 2023.

\bibitem{zamani2024data}
S.~Zamani~Ashtiani, {\em Data Compression, Uncertainty Quantification, and Prediction Using Low-Rank Approximation}.
\newblock PhD thesis, University of Pittsburgh, 2024.

\bibitem{ashtiani2022scalable}
S.~Z. Ashtiani, M.~R. Malik, and H.~Babaee, ``Scalable in situ compression of transient simulation data using time-dependent bases,'' {\em Journal of Computational Physics}, vol.~468, p.~111457, 2022.

\bibitem{sharkawy2022forward}
A.-N. Sharkawy, ``Forward and inverse kinematics solution of a robotic manipulator using a multilayer feedforward neural network,'' {\em Journal of Mechanical and Energy Engineering}, vol.~6, 2022.

\bibitem{singh2024application}
R.~Singh, A.~Agrawal, A.~Mishra, P.~K. Arya, A.~Sharma, {\em et~al.}, ``Application of deep learning model for analysis of forward kinematics of a 6-axis robotic hand for a humanoid,'' in {\em 2024 3rd International Conference on Artificial Intelligence and Autonomous Robot Systems (AIARS)}, pp.~1--6, IEEE, 2024.

\bibitem{toquica2021analytical}
J.~S. Toquica, P.~S. Oliveira, W.~S. Souza, J.~M.~S. Motta, and D.~L. Borges, ``An analytical and a deep learning model for solving the inverse kinematic problem of an industrial parallel robot,'' {\em Computers \& Industrial Engineering}, vol.~151, p.~106682, 2021.

\bibitem{lu2024robust}
T.~Lu, K.~Zhang, and Y.~Shi, ``Robust data-driven model predictive control via on-policy reinforcement learning for robot manipulators,'' in {\em 2024 IEEE 7th International Conference on Industrial Cyber-Physical Systems (ICPS)}, pp.~1--6, IEEE, 2024.

\bibitem{anand2023model}
A.~S. Anand, J.~T. Gravdahl, and F.~J. Abu-Dakka, ``Model-based variable impedance learning control for robotic manipulation,'' {\em Robotics and Autonomous Systems}, vol.~170, p.~104531, 2023.

\bibitem{ngo2024robust}
T.~Q. Ngo and T.~H. Tran, ``Robust adaptive iterative learning control for de-icing robot manipulator,'' {\em Journal of Robotics and Control (JRC)}, vol.~5, no.~3, pp.~746--755, 2024.

\bibitem{calandra2016manifold}
R.~Calandra, J.~Peters, C.~E. Rasmussen, and M.~P. Deisenroth, ``Manifold gaussian processes for regression,'' in {\em 2016 International joint conference on neural networks (IJCNN)}, pp.~3338--3345, IEEE, 2016.

\bibitem{carron2019data}
A.~Carron, E.~Arcari, M.~Wermelinger, L.~Hewing, M.~Hutter, and M.~N. Zeilinger, ``Data-driven model predictive control for trajectory tracking with a robotic arm,'' {\em IEEE Robotics and Automation Letters}, vol.~4, no.~4, pp.~3758--3765, 2019.

\bibitem{chen2018neural}
R.~T. Chen, Y.~Rubanova, J.~Bettencourt, and D.~K. Duvenaud, ``Neural ordinary differential equations,'' {\em Advances in neural information processing systems}, vol.~31, 2018.

\bibitem{chee2022knode}
K.~Y. Chee, T.~Z. Jiahao, and M.~A. Hsieh, ``Knode-mpc: A knowledge-based data-driven predictive control framework for aerial robots,'' {\em IEEE Robotics and Automation Letters}, vol.~7, no.~2, pp.~2819--2826, 2022.

\bibitem{coulson2019data}
J.~Coulson, J.~Lygeros, and F.~D{\"o}rfler, ``Data-enabled predictive control: In the shallows of the deepc,'' in {\em 2019 18th European Control Conference (ECC)}, pp.~307--312, IEEE, 2019.

\bibitem{willems2005note}
J.~C. Willems, P.~Rapisarda, I.~Markovsky, and B.~L. De~Moor, ``A note on persistency of excitation,'' {\em Systems \& Control Letters}, vol.~54, no.~4, pp.~325--329, 2005.

\bibitem{willems1997introduction}
J.~C. Willems and J.~W. Polderman, {\em Introduction to mathematical systems theory: a behavioral approach}, vol.~26.
\newblock Springer Science \& Business Media, 1997.

\bibitem{NL-DEEPC1}
K.~Zhang, K.~Chen, X.~Lin, Y.~Zheng, X.~Yin, X.~Hu, Z.~Song, and Z.~Li, ``Data-enabled predictive control for fast charging of lithium-ion batteries with constraint handling,'' 2023.

\bibitem{NL-DEEPC2}
H.~Wang, K.~Zhang, K.~Lee, Y.~Mei, K.~Zhu, V.~Srivastava, J.~Sheng, and Z.~Li, ``Mechanical design and data-enabled predictive control of a planar soft robot,'' {\em IEEE Robotics and Automation Letters}, vol.~9, no.~9, pp.~7923--7930, 2024.

\bibitem{vahidi2023data}
A.~Vahidi-Moghaddam, K.~Zhang, Z.~Li, and Y.~Wang, ``Data-enabled neighboring extremal optimal control: A computationally efficient deepc,'' in {\em 2023 62nd IEEE Conference on Decision and Control (CDC)}, pp.~4778--4783, IEEE, 2023.

\bibitem{huang2023robust}
L.~Huang, J.~Zhen, J.~Lygeros, and F.~D{\"o}rfler, ``Robust data-enabled predictive control: Tractable formulations and performance guarantees,'' {\em IEEE Transactions on Automatic Control}, vol.~68, no.~5, pp.~3163--3170, 2023.

\bibitem{ghaemi2009neighboring}
R.~Ghaemi, J.~Sun, and I.~V. Kolmanovsky, ``Neighboring extremal solution for nonlinear discrete-time optimal control problems with state inequality constraints,'' {\em IEEE Transactions on Automatic Control}, vol.~54, no.~11, pp.~2674--2679, 2009.

\bibitem{ghaemi2010robust}
R.~Ghaemi, {\em Robust Model Based Control of Constrained Systems.}
\newblock PhD thesis, 2010.

\bibitem{ghaemi2008neighboring}
R.~Ghaemi, J.~Sun, and I.~Kolmanovsky, ``Neighboring extremal solution for discrete-time optimal control problems with state inequality constraints,'' in {\em 2008 American Control Conference}, pp.~3823--3828, IEEE, 2008.

\bibitem{ozgur2016kinematic}
E.~{\"O}zg{\"u}r and Y.~Mezouar, ``Kinematic modeling and control of a robot arm using unit dual quaternions,'' {\em Robotics and Autonomous Systems}, vol.~77, pp.~66--73, 2016.

\bibitem{vahidi2024online}
A.~Vahidi-Moghaddam, K.~Zhang, X.~Yin, V.~Srivastava, and Z.~Li, ``Online reduced-order data-enabled predictive control,'' {\em arXiv preprint arXiv:2407.16066}, 2024.

\end{thebibliography}
\end{document}